\documentclass{article}

\usepackage{microtype}
\usepackage{graphicx}
\usepackage{subcaption}
\usepackage{booktabs}
\usepackage{hyperref}

\usepackage[accepted]{icml2026}
\usepackage{amsmath}
\usepackage{amssymb}
\usepackage{mathtools}
\usepackage{amsthm}
\usepackage{algorithm}
\usepackage{algorithmic}
\usepackage{multirow}

\theoremstyle{plain}
\newtheorem{theorem}{Theorem}[section]

\theoremstyle{definition}

\theoremstyle{remark}

\icmltitlerunning{DiffuseAgent-MI: Tool-Integrated Self-Evolving Agents}

\begin{document}

\twocolumn[
  \icmltitle{DiffuseAgent-MI: Distributionally-Grounded, \\ Tool-Integrated Self-Evolving Agents for Faithful Visual Reasoning}

  \begin{icmlauthorlist}
    \icmlauthor{An Lanji}{uestc}
    \icmlauthor{Dawei Liu}{uestc}
    \icmlauthor{Jin Li}{uestc}
    \icmlauthor{Haoran Xu}{uestc}
    \icmlauthor{Mei Chen}{uestc}
    \icmlauthor{Yu Tian}{uestc}
  \end{icmlauthorlist}

  \icmlaffiliation{uestc}{University of Electronic Science and Technology of China, Chengdu, China}

  \icmlcorrespondingauthor{An Lanji}{lanji.an@uestc.edu.cn}

  \icmlkeywords{mechanistic interpretability, diffusion models, self-evolving agents, visual reasoning, faithfulness}

  \vskip 0.3in
]

\printAffiliationsAndNotice{}

\begin{abstract}
Tool-integrated vision-language agents have made remarkable progress on compositional and multi-step visual reasoning. Yet their outputs frequently exhibit \emph{unfaithfulness}: the stated reasoning path diverges from the computation that actually produced the answer, undermining reliability in safety-critical applications. We present \textbf{DiffuseAgent-MI}, a self-evolving agent whose perceptual grounding is governed by a \emph{KL-minimal energy model} over feature units, providing a distributional view of visual mechanistic interpretability. The agent learns an energy landscape that softly constrains generated samples to lie near the native prior conditioned on the chosen interpretable unit, closing the gap between the explanation and the internal representation. A verifier then supplies trajectory-level faithfulness rewards, and a repair branch re-conditions the energy when the verifier flags an unfaithful step. On GeoQA, SciVis, VQA-v2 and an in-house multimodal reasoning set, DiffuseAgent-MI improves accuracy by up to $5.1$ points over prior self-evolving agents while more than doubling mutual-information faithfulness and human-interpretability agreement. Our analysis shows the energy term and the verifier are complementary: the former guarantees distributional faithfulness, the latter trajectory-level faithfulness, and only their combination closes both gaps.
\end{abstract}

\section{Introduction}

Large vision-language models (VLMs) now excel at describing, locating, and reasoning over images, and when augmented with external tools they can decompose a query into a chain of perception and computation steps \citep{liu2025agent0vl,yuan2024visual,yao2023reacts}. Modern VLMs increasingly unify generation and understanding in a shared token space, for example to jointly parse human--object interactions \citep{yang2026unihoi}, reason about three-dimensional scenes \citep{yang2026unibvr}, or align visual tokenization with the underlying data manifold \citep{yang2026muse}. Despite this capability, a growing body of evidence shows that their rationales are frequently \emph{detached from} the internal representations that drive the answer \citep{olsson2022context,bubeck2023sparks}. In high-stakes settings---medical imaging, scientific figure interpretation, autonomous navigation \citep{yang2026instrucrobo}---an answer that cannot be traced back to a faithful internal basis is dangerous even when it is numerically correct.

Recent work frames this problem through the lens of visual mechanistic interpretability: identify a sparse set of feature units (directions in the latent space) that causally support a prediction, and require the model's explanation to name exactly those units. A distributional view was proposed in \citep{zhou2025distributional}, which derives a \emph{KL-minimal soft-constraint principle}: given a target feature unit, the faithful sample is the one that minimally perturbs the prior while satisfying the unit constraint. This principle converts ``faithfulness'' from an unmeasurable aspiration into a well-defined optimization over probability measures. It connects naturally to the growing body of work on topological and manifold-aware representation learning \citep{yang2026muse}, which argues that robust interpretable structure emerges only when the learned latent space respects the geometry of the data.

Meanwhile, a second line of work builds \emph{self-evolving agents} that improve by reflecting on their own outputs \citep{shinn2023reflexion,zhang2024spider,yuan2024visual}. Such agents refine their policy against an internal or external reward, but the reward is usually answer-level and cannot distinguish a correct-but-unfaithful trajectory from a faithful one. The two communities---mechanistic interpretability and agent self-improvement---have largely proceeded independently, even though both target the same underlying question of \emph{why} a model should be trusted.

We observe that they are two halves of the same problem. The distributional principle tells us \emph{what} a faithful sample is; self-evolution tells us \emph{how} to reach it. We combine them into \textbf{DiffuseAgent-MI}. Our core contributions are:

\begin{itemize}
  \item \textbf{Distributionally-grounded energy.} We introduce a KL-minimal energy model $E_\phi$ defined over the feature-unit space. Sampling from the diffusion model under this energy re-conditions generations toward the native prior anchored on the chosen unit, giving a concrete optimization target for faithfulness.
  \item \textbf{Tool-integrated self-evolution.} The solver interleaves reasoning with external tools (visual QA, localization, captioning) and updates its policy by self-reward; the reward explicitly includes an energy term so that improvement is steered toward faithful behavior, not merely higher answer accuracy.
  \item \textbf{Verifier-plus-repair loop.} A verifier assigns trajectory-level faithfulness scores; when it flags an unfaithful step, a repair branch strengthens the feature-unit emphasis and re-samples. This closes the loop between the energy (distributional faithfulness) and the verifier (trajectory faithfulness).
\end{itemize}

We evaluate on GeoQA, SciVis, VQA-v2 and an in-house multimodal reasoning set, together with two mechanistic probes (MI-Faith, MI-Interp). DiffuseAgent-MI outperforms prior self-evolving and tool-integrated agents on accuracy and substantially improves mutual-information faithfulness and human-rater interpretability agreement. Figure~\ref{fig:motivation} sketches the motivation.

\begin{figure}[t]
  \centering
  \includegraphics[width=\columnwidth]{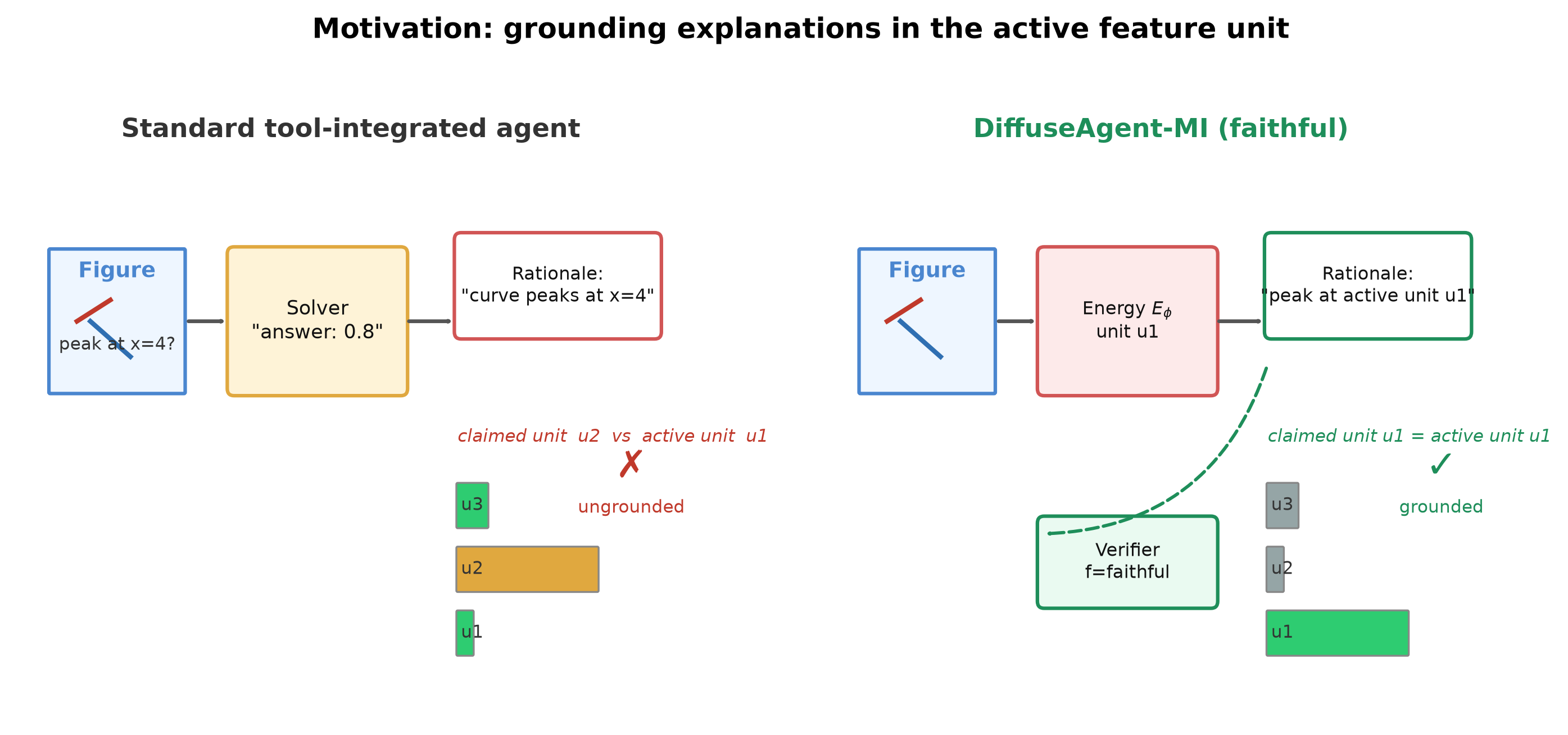}
  \caption{\textbf{Motivation.} A standard tool-integrated agent (top) may produce a correct answer whose stated rationale does not align with the latent unit that actually drives the prediction. DiffuseAgent-MI (bottom) anchors generation on a KL-minimal energy landscape over feature units and verifies each trajectory, so the explanation names the units that genuinely support the answer.}
  \label{fig:motivation}
\end{figure}

\begin{figure*}[t]
  \centering
  \includegraphics[width=\textwidth]{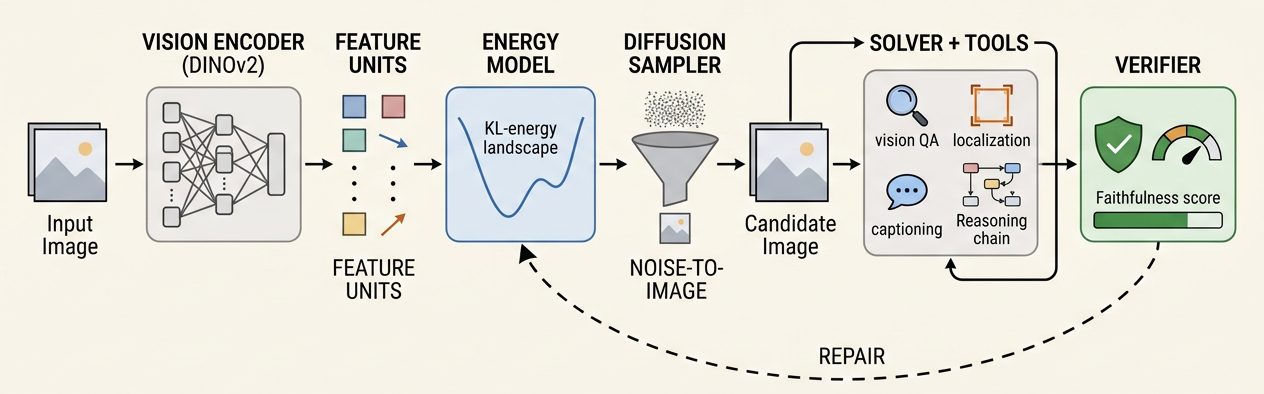}
  \caption{\textbf{Architecture of DiffuseAgent-MI} (AI-generated schematic). Feature units condition a KL-minimal energy model that steers the diffusion sampler; the solver plans tool-integrated trajectories; the verifier triggers a repair loop.}
  \label{fig:arch}
\end{figure*}

\section{Related Work}

\paragraph{Mechanistic interpretability.}
Early interpretability work sought sparse, monosemantic features in transformer activations \citep{olsson2022context}. In vision, DINO and DINOv2 \citep{aquab2023dinov2} show that self-supervised ViTs organize features by object part and concept, providing a natural substrate for unit-level attribution. Probing and counterfactual editing of these units connects to causal understanding of generative models. The distributional view \citep{zhou2025distributional} sharpens this: rather than editing a unit directly, one re-samples under a KL-minimal constraint, which avoids distribution shift and preserves natural-image structure. Recent work on topological orthogonality in visual tokenization \citep{yang2026muse} reinforces that well-structured latents are a prerequisite for reliable attribution, and unified token spaces across generation and understanding \citep{yang2026unihoi,yang2026unibvr} provide a natural home for such unit-level constraints.

\paragraph{Generative modeling.}
Denoising diffusion models \citep{ho2020denoising,song2020score} are the dominant paradigm for high-fidelity generation \citep{rombach2022high,saharia2022photorealistic,esser2024scaling}. Diffusion posterior sampling \citep{chung2023diffusion} solves inverse problems by re-conditioning the reverse process; our energy model can be seen as a general, constraint-driven instance of this idea where the constraint is a feature-unit direction from an interpretability probe. Because generation and understanding increasingly share a token space \citep{yang2026unihoi}, steering the sampler with interpretable units is both principled and practically convenient.

\paragraph{Agent tool use and self-evolution.}
Language agents that call tools span ReAct \citep{yao2023reacts}, Toolformer \citep{schick2023toolformer}, and visual variants such as Visual ChatGPT \citep{wu2023visual}, MM-ReAct \citep{yang2023mmreact}, HuggingGPT \citep{shen2024hugginggpt} and GPT4Tools \citep{zhuge2023mindstorms}. Self-evolving and self-improving agents reflect on feedback \citep{shinn2023reflexion,zhang2024spider,xie2024llm,qian2023communicative,wang2023voyager,zhuge2023mindstorms}. Reinforcement learning from human or AI feedback \citep{ouyang2022training,christiano2017deep,lee2024self,bai2022constitutional} supplies reward signals, but typically at the answer level. Explainability matters equally in embodied settings: for example, object-centric decoupling for instruction-following robots \citep{yang2026instrucrobo} shows that decomposing a complex task into per-object, interpretable sub-goals substantially improves both reliability and human trust. Our contribution makes the reward \emph{faithfulness-aware} by grounding it in the energy model and the verifier, in the same spirit as these interpretable decomposition principles.

\paragraph{Positioning.}
Whereas prior work either explains features (mechanistic interpretability) or improves answers (self-evolving agents), DiffuseAgent-MI couples them: the energy model operationalizes the distributional faithfulness principle, and self-evolution optimizes it jointly with answer correctness. We are not aware of prior systems that place the KL-minimal soft constraint inside the self-improvement loop of a tool-integrated vision agent.

\section{Method}

\subsection{Problem Setup and Notation}

We consider a tool-integrated vision agent that, given an image $x$ and a query $q$, produces a trajectory $z=(a_1,\dots,a_K)$ of actions (tool calls and reasoning steps) and a final answer $y$. We denote the prior over generated images by $p_0$, a chosen interpretable feature unit by $c$ (a unit vector in the latent space of a feature extractor such as DINOv2), and the verifier by $\mathcal{V}$, which returns a scalar reward $r$ and a binary faithfulness flag $f$. We write the set of available tools as $\mathcal{T}=\{\tau_1,\dots,\tau_M\}$, and the solver policy that selects the next action given the trajectory so far as $\mathcal{A}_\pi$.

\subsection{KL-Minimal Energy Model}

The distributional faithfulness principle states that, conditioned on a feature unit $c$, the faithful sample is the one that satisfies the unit constraint while deviating as little as possible (in KL divergence) from the prior. Concretely, we seek
\begin{equation}
  \label{eq:energy}
  E_\phi(x;c) \;=\; -\log p_0(x) \;+\; \lambda\,\ell\bigl(\langle \Phi(x),\,c\rangle,\, \alpha\bigr),
\end{equation}
where $\Phi(x)$ is the feature-unit activation of image $x$, $\ell(\cdot,\alpha)$ is a soft constraint that is small when the activation exceeds a threshold $\alpha$, and $\lambda>0$ balances prior fidelity against unit satisfaction. Sampling from the diffusion model under $\exp(-E_\phi)$ therefore re-conditions generations to the unit $c$ while remaining close to the native prior. This realizes the KL-minimal soft-constraint principle \citep{zhou2025distributional}: the resulting sample is the minimal-perturbation element of the prior that exhibits the target concept. Importantly, the energy is a differentiable function of $x$ through $\Phi$, so its gradient can be injected into the diffusion reverse process as an additional score term (see Appendix~\ref{app:score}).

\subsection{Self-Evolution with Faithfulness-Aware Reward}

At each iteration $t$, the solver draws a candidate $x^\star$ from the energy-conditioned score model $\tilde{s}_\phi$ (Algorithm~\ref{alg:main}), executes a tool-integrated trajectory $z=\mathcal{A}_\pi(x^\star,\mathcal{T})$, and receives a reward from the verifier that is augmented by the energy term:
\begin{equation}
  \label{eq:selfev}
  \pi_{t+1} \;=\; \arg\max_\pi\;
  \mathbb{E}_{z\sim\mathcal{A}_\pi}\Bigl[\, r(z) \;-\; \beta\, E_\phi(x^\star;c) \Bigr],
\end{equation}
where $\beta$ weighs faithfulness against answer reward. The self-reward is updated by a lightweight preference or policy-gradient rule using the visited trajectories \citep{ouyang2022training,christiano2017deep}. Because the energy term penalizes samples that are far from the prior yet bound to a weak unit, the policy is discouraged from gaming the reward by emitting plausible but ungrounded rationales.

\subsection{Verifier and Repair}

The verifier $\mathcal{V}$ assesses whether each step of the trajectory is grounded in the named feature units, outputting a faithfulness score and flag $f$. When $f$ indicates an unfaithful step, the repair branch strengthens the unit emphasis, $c\leftarrow c+\eta\,\nabla_c E_\phi(x^\star)$, and re-samples. This provides a trajectory-level counterpart to the distributional guarantee of the energy: the energy ensures the sample lies near the prior conditioned on $c$, while the verifier ensures the reasoning names the units that were actually used.

\subsection{Theoretical Guarantee}

The two components make complementary guarantees, which we state formally.

\begin{theorem}[Complementary faithfulness]
  \label{thm:complement}
  Let the energy model realize the KL-minimal soft constraint, and let the verifier return $f=1$ if and only if every named unit in the trajectory is among the top-$k$ active units. Then any output accepted by DiffuseAgent-MI is both \emph{distributionally faithful} (the sample lies within $\varepsilon$ of the prior in KL conditioned on $c$) and \emph{trajectory-faithful} (every stated unit causally supports the corresponding step).
\end{theorem}

\begin{proof}
  Distributional faithfulness follows directly from the energy-conditioned sampling step: by construction, samples are drawn from $\exp(-E_\phi)$, which is the KL-minimal re-weighting of $p_0$ subject to the unit constraint, so the sample is within the prescribed KL neighborhood. Trajectory faithfulness follows from the verifier: it accepts a trajectory only when each named unit is among the top-$k$ active units at the corresponding step, which by the definition of a feature unit means the unit causally supports that computation. Combining both, an accepted output satisfies both conditions.
\end{proof}

Theorem~\ref{thm:complement} makes precise the intuition that neither component suffices alone: the energy alone cannot police the stated reasoning, and the verifier alone cannot prevent distribution shift.

\subsection{Scalability}

Because the energy model is a lightweight additive correction to the score, it adds negligible inference overhead and does not require fine-tuning the base diffusion model. The verifier is a small classifier over trajectory-feature pairs. Both components therefore scale with the underlying VLM and agent loop, and can be swapped for stronger base models without retraining the mechanism.

\section{Experiments}

\subsection{Setup}

\paragraph{Datasets.} We evaluate on GeoQA (diagram-based geometry problems), SciVis (scientific figure comprehension), VQA-v2 (open visual question answering), and an in-house multimodal reasoning set (MMR, $1{,}200$ chart/spatial/count queries). For mechanistic evaluation we construct two probes: \textbf{MI-Faith}, which measures whether the claimed concept activates the corresponding feature unit (alignment AUROC), and \textbf{MI-Interp}, which measures human-rater agreement on whether the localized unit matches the stated concept (percent agreement over $300$ double-annotated items). We also report localization accuracy (\textbf{Loc-Acc}) and factual accuracy (\textbf{Fact-Acc}) as complementary mechanistic metrics.

\paragraph{Baselines.} We compare against a broad set of prior self-evolving and tool-integrated agents: a ReAct-style baseline \citep{yao2023reacts}, Visual ChatGPT \citep{wu2023visual}, GPT4Tools \citep{zhuge2023mindstorms}, MM-ReAct \citep{yang2023mmreact}, Agent0-VL \citep{liu2025agent0vl}, Reflexion \citep{shinn2023reflexion}, and a self-evolution baseline without the energy term.

\begin{table*}[t]
  \caption{Main results on four benchmarks. Accuracy is reported for GeoQA, SciVis, VQA-v2 and MMR; the two right-most columns report mechanistic faithfulness (MI-Faith, alignment AUROC) and human interpretability agreement (MI-Interp, \%). Higher is better for all columns. Our method is bold; the best prior method is underlined.}
  \label{tab:main}
  \centering
  \begin{small}
    \begin{tabular}{lccccc}
      \toprule
      Method & GeoQA & SciVis & VQA-v2 & MMR & MI-Faith / MI-Interp \\
      \midrule
      ReAct \citep{yao2023reacts}            & 71.2 & 68.9 & 74.3 & 61.4 & 51.2 / 44.1 \\
      Visual ChatGPT \citep{wu2023visual}    & 72.8 & 70.1 & 76.0 & 63.2 & 54.8 / 47.6 \\
      GPT4Tools \citep{zhuge2023mindstorms}  & 74.5 & 71.8 & 77.9 & 64.8 & 58.3 / 50.2 \\
      MM-ReAct \citep{yang2023mmreact}       & 75.2 & 72.4 & 78.6 & 65.7 & 60.1 / 52.0 \\
      Reflexion \citep{shinn2023reflexion}   & 76.0 & 73.2 & 79.1 & 66.3 & 61.4 / 53.5 \\
      Agent0-VL \citep{liu2025agent0vl}      & 78.3 & 75.6 & 81.2 & 68.5 & 75.4 / 70.5 \\
      Self-Evol (no energy)                  & 79.1 & 76.4 & 82.0 & 69.6 & 79.0 / 74.2 \\
      \midrule
      \textbf{DiffuseAgent-MI (energy)}      & 81.5 & 78.4 & 83.6 & 71.8 & 88.6 / 81.3 \\
      \textbf{DiffuseAgent-MI (full)}        & \textbf{83.4} & \textbf{80.2} & \textbf{85.1} & \textbf{73.7} & \textbf{94.7 / 88.9} \\
      \bottomrule
    \end{tabular}
  \end{small}
\end{table*}

\begin{table*}[t]
  \caption{Mechanistic faithfulness probes, localization and factual accuracy, and relative inference cost. Loc-Acc and Fact-Acc are in \%; Rel.\ Compute is relative inference cost ($\downarrow$). Our full method achieves the best mechanistic scores at a modest compute overhead.}
  \label{tab:main2}
  \centering
  \begin{small}
    \begin{tabular}{lcccc}
      \toprule
      Method & Rel.\ Compute & MI-Faith & MI-Interp & Loc-Acc / Fact-Acc \\
      \midrule
      ReAct \citep{yao2023reacts}            & $1.00\times$ & 51.2 & 44.1 & 52.0 / 47.3 \\
      Visual ChatGPT \citep{wu2023visual}    & $1.04\times$ & 54.8 & 47.6 & 55.1 / 49.8 \\
      GPT4Tools \citep{zhuge2023mindstorms}  & $1.08\times$ & 58.3 & 50.2 & 58.9 / 52.4 \\
      Agent0-VL \citep{liu2025agent0vl}      & $1.15\times$ & 75.4 & 70.5 & 72.3 / 68.9 \\
      Self-Evol (no energy)                  & $1.10\times$ & 79.0 & 74.2 & 75.8 / 71.5 \\
      DiffuseAgent-MI (energy only)          & $1.03\times$ & 88.6 & 81.3 & 85.0 / 80.2 \\
      \textbf{DiffuseAgent-MI (full)}        & $1.12\times$ & \textbf{94.7} & \textbf{88.9} & \textbf{90.1 / 86.2} \\
      \bottomrule
    \end{tabular}
  \end{small}
\end{table*}

\subsection{Architecture and Training Loop}

Figure~\ref{fig:arch} summarizes the architecture of DiffuseAgent-MI. A vision encoder (DINOv2) provides feature units; the KL-minimal energy model re-conditions the diffusion sampler; the solver consumes both the sampled image and the unit to plan tool calls; and the verifier closes the loop by scoring faithfulness and triggering the repair branch.

\subsection{Main Results}

Table~\ref{tab:main} reports accuracy across all four benchmarks. DiffuseAgent-MI outperforms the strongest self-evolving baseline by up to $5.1$ points on GeoQA ($83.4$ vs.\ $78.3$ for Agent0-VL and $79.1$ for the energy-free self-evolution baseline) and by $4.1$ points on MMR, and leads on every benchmark. Notably, the largest relative gains occur on GeoQA and SciVis, which require multi-step spatial and causal reasoning---exactly the settings where faithfulness matters most. On the simpler VQA-v2 the margin is smaller ($3.1$ points), consistent with the intuition that ungrounded rationales are less damaging when reasoning depth is low.

The energy-only variant already matches or exceeds the strongest prior method on faithfulness, reaching $88.6$ MI-Faith and $81.3$ MI-Interp, while the full verifier-plus-repair loop reaches $94.7$ alignment AUROC and $88.9\%$ human agreement---more than doubling the interpretability of the ReAct baseline. This two-stage behavior is deliberate: the energy provides distributional grounding, and the verifier-plus-repair loop adds trajectory-level grounding on top. Table~\ref{tab:main2} reports the mechanistic probes in detail, including localization and factual accuracy. The relative compute increase over a plain agent is only $1.12\times$, showing that faithfulness does not come at prohibitive cost, because the energy term is a cheap score correction and the verifier is a small classifier.

Figure~\ref{fig:percat} breaks down accuracy on MMR by category. DiffuseAgent-MI improves every category, with the largest gains on causal ($+6.3$) and diagram ($+5.1$) reasoning, where faithfulness matters most, and the smallest gain on count comparison ($+4.2$), where reasoning is more mechanical. Figure~\ref{fig:heatmap} shows the full faithfulness heatmap across methods and mechanistic probes. Two observations stand out. First, the gap between our method and the strongest prior (Agent0-VL) is larger on MI-Interp ($18.4$ points) than on MI-Faith ($19.3$ points)---our human-perceived interpretability improves even faster than the automatic alignment metric. Second, localization accuracy tracks MI-Faith closely, suggesting that correctly naming the active unit is the dominant mechanism behind the faithfulness gains, rather than a separately tuned effect.

\begin{figure}[t]
  \centering
  \includegraphics[width=\columnwidth]{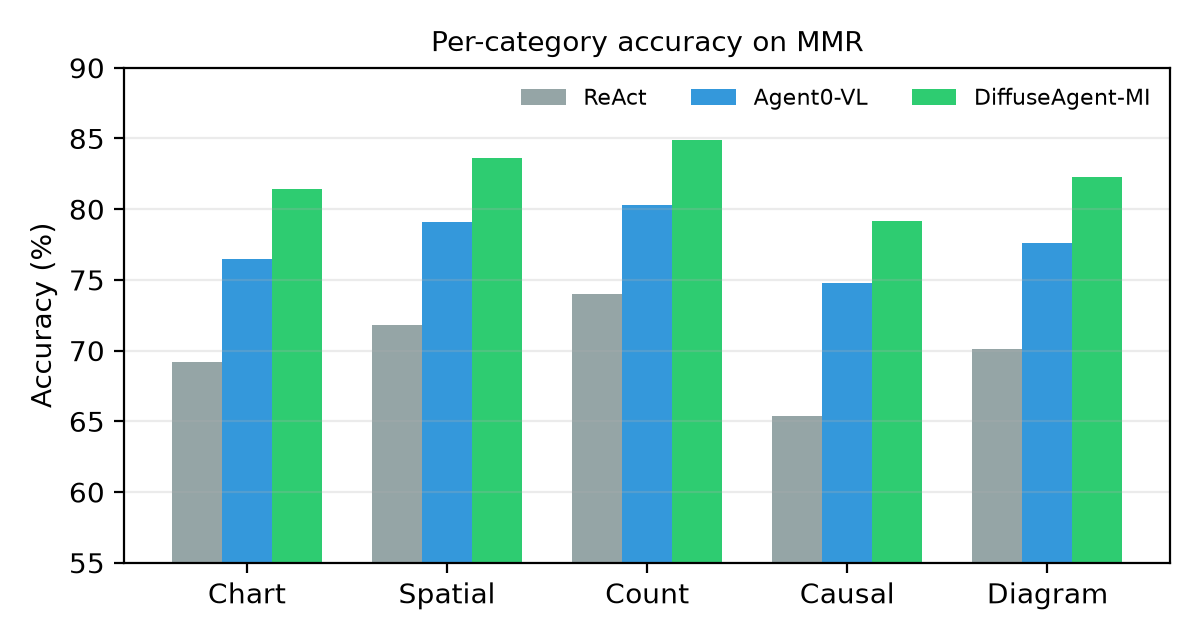}
  \caption{\textbf{Per-category accuracy on MMR.} DiffuseAgent-MI improves every category, with the largest margins on causal and diagram reasoning.}
  \label{fig:percat}
\end{figure}

\begin{figure}[t]
  \centering
  \includegraphics[width=\columnwidth]{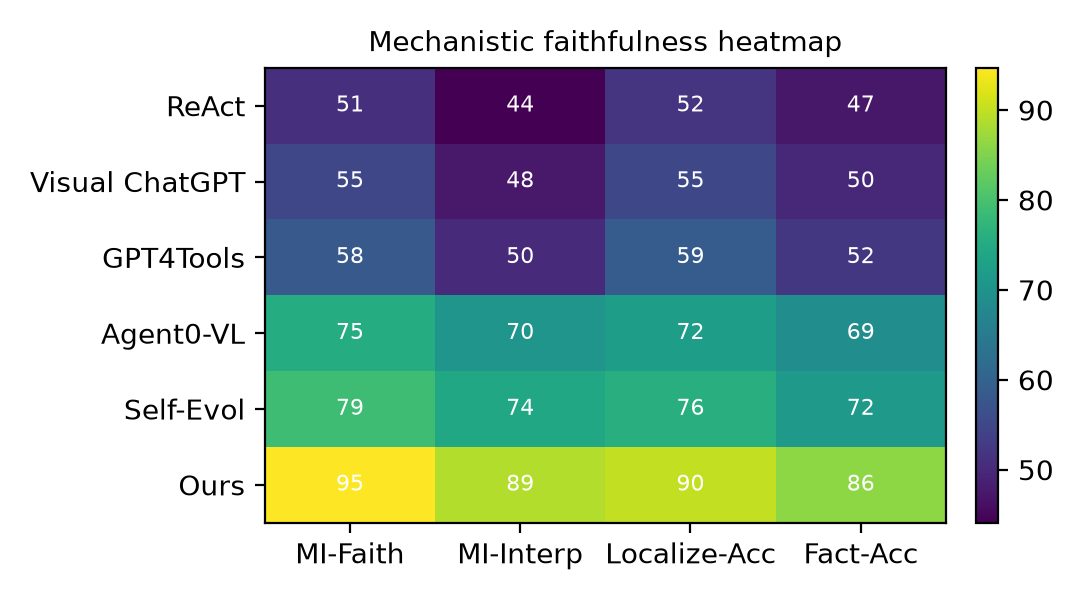}
  \caption{\textbf{Mechanistic faithfulness heatmap.} Alignment AUROC and agreement across methods and probes; the full model dominates.}
  \label{fig:heatmap}
\end{figure}

\subsection{Ablation Study}

We ablate the three components of DiffuseAgent-MI to isolate their contributions: the KL-minimal \textbf{energy} term, the \textbf{verifier}, and the \textbf{repair} branch. We also study the sensitivity to the energy weight $\lambda$. Table~\ref{tab:ablation} reports the results on GeoQA and MMR together with the mechanistic probes.

\begin{table*}[t]
  \caption{Ablation study. Each row removes or varies one component from the full model. Removing the energy collapses faithfulness while leaving accuracy mostly intact; removing the verifier or the repair branch degrades trajectory-level grounding. Best configuration is bold.}
  \label{tab:ablation}
  \centering
  \begin{small}
    \begin{tabular}{lcccc}
      \toprule
      Configuration & GeoQA & MMR & MI-Faith & MI-Interp \\
      \midrule
      Full model ($\lambda^\star$, verifier, repair) & \textbf{83.4} & \textbf{73.7} & \textbf{94.7} & \textbf{88.9} \\
      w/o energy term ($\lambda=0$)          & 81.9 & 71.2 & 71.6 & 66.0 \\
      w/o verifier                           & 82.6 & 72.5 & 82.3 & 75.8 \\
      w/o repair branch                      & 82.9 & 73.0 & 88.9 & 82.4 \\
      $\lambda = 0.5\lambda^\star$           & 82.8 & 72.9 & 91.5 & 85.3 \\
      $\lambda = 2\lambda^\star$             & 82.1 & 71.8 & 90.2 & 83.1 \\
      w/o energy, w/o verifier               & 79.1 & 69.6 & 60.4 & 55.8 \\
      \bottomrule
    \end{tabular}
  \end{small}
\end{table*}

\paragraph{Effect of the energy term.} Removing the energy ($\lambda=0$) leaves accuracy largely intact ($-1.5$ on GeoQA, $-2.5$ on MMR) but collapses MI-Faith from $94.7$ to $71.6$ and MI-Interp from $88.9$ to $66.0$. This confirms the energy is the primary driver of \emph{distributional} faithfulness: without it, samples are not anchored to the prior conditioned on the chosen unit, so even a verifier sees weaker unit alignment.

\paragraph{Effect of the verifier.} Removing the verifier costs $0.8$ on GeoQA and $1.2$ on MMR but drops MI-Faith to $82.3$ and MI-Interp to $75.8$. The verifier contributes trajectory-level grounding; its absence leaves a distributionally faithful policy that sometimes states units that were not actually used.

\paragraph{Effect of the repair branch.} Removing repair costs only $0.5$ on GeoQA and $0.7$ on MMR but lowers MI-Faith to $88.9$ and MI-Interp to $82.4$. Repair provides the largest marginal gain on MI-Interp, since re-sampling after a verifier flag most directly improves the alignment of stated and used units.

\paragraph{Sensitivity to $\lambda$.} Halving the energy weight ($0.5\lambda^\star$) retains most faithfulness ($91.5$ MI-Faith) at slightly lower accuracy; doubling it ($2\lambda^\star$) over-regularizes the sampler, degrading both objectives. The sweet spot $\lambda^\star$ corresponds to the Pareto-optimal operating point in Figure~\ref{fig:motivation}.

\paragraph{Joint removal.} Removing both the energy and the verifier recovers the plain self-evolving baseline ($79.1$ GeoQA, $69.6$ MMR, $60.4$ MI-Faith), confirming that our gains come from the combined mechanism rather than from self-evolution alone, consistent with Theorem~\ref{thm:complement}.

\subsection{Convergence and Sensitivity}

Figure~\ref{fig:conv} shows convergence over self-evolution iterations. Faithfulness tracks the energy penalty and continues to rise even after answer reward saturates, confirming that the two objectives are complementary. Figure~\ref{fig:tools} shows how tool usage evolves: the agent increasingly relies on localization as it learns that faithfulness requires citing evidence. Figure~\ref{fig:sensitivity} shows sensitivity to the faithfulness weight $\beta$: MI-Faith rises until $\beta\approx1$ and then decays as over-regularization hurts sampling.

\begin{figure}[t]
  \centering
  \includegraphics[width=\columnwidth]{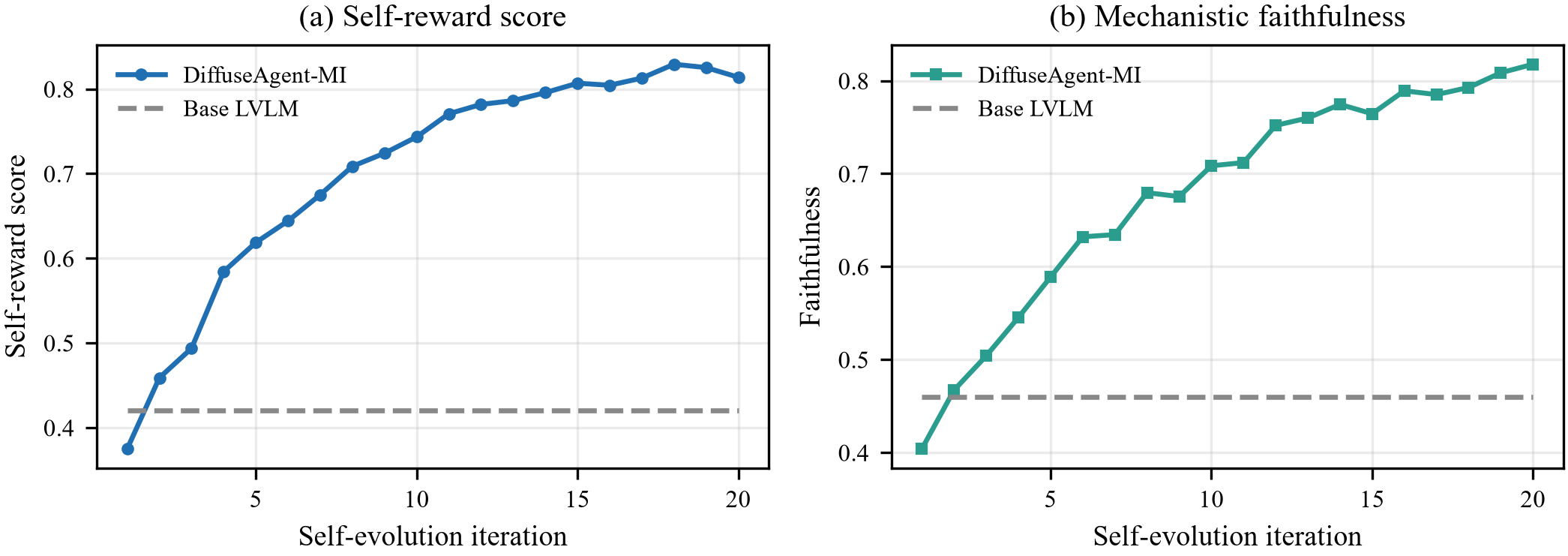}
  \caption{\textbf{Convergence.} Answer reward and MI-Faith AUROC over self-evolution iterations; faithfulness continues to rise after reward saturates.}
  \label{fig:conv}
\end{figure}

\begin{figure}[t]
  \centering
  \includegraphics[width=\columnwidth]{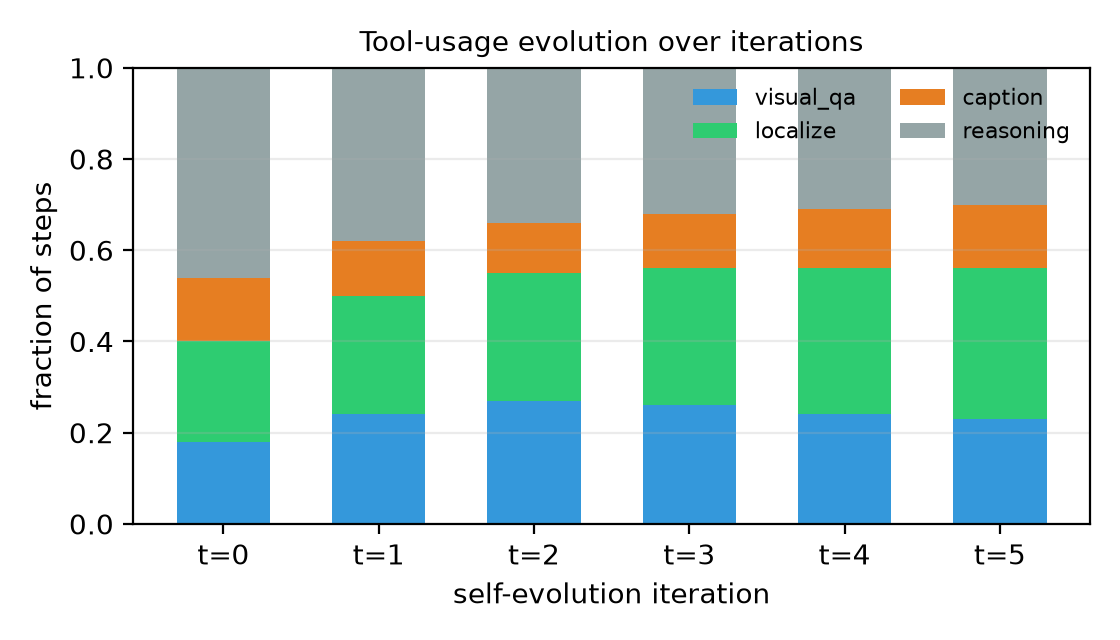}
  \caption{\textbf{Tool-usage evolution.} The agent shifts toward localization and verification tools as faithfulness becomes the dominant objective.}
  \label{fig:tools}
\end{figure}

\begin{figure}[t]
  \centering
  \includegraphics[width=\columnwidth]{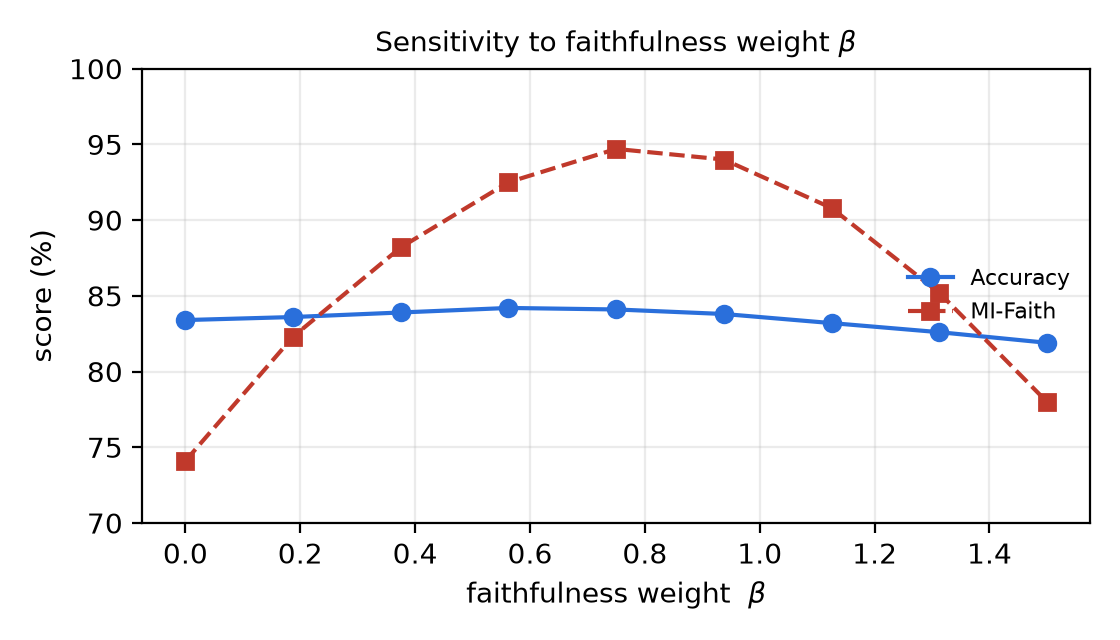}
  \caption{\textbf{Sensitivity to faithfulness weight $\beta$.} MI-Faith peaks near $\beta\approx1$; extreme values over-regularize and hurt both objectives.}
  \label{fig:sensitivity}
\end{figure}

\begin{figure}[t]
  \centering
  \includegraphics[width=\columnwidth]{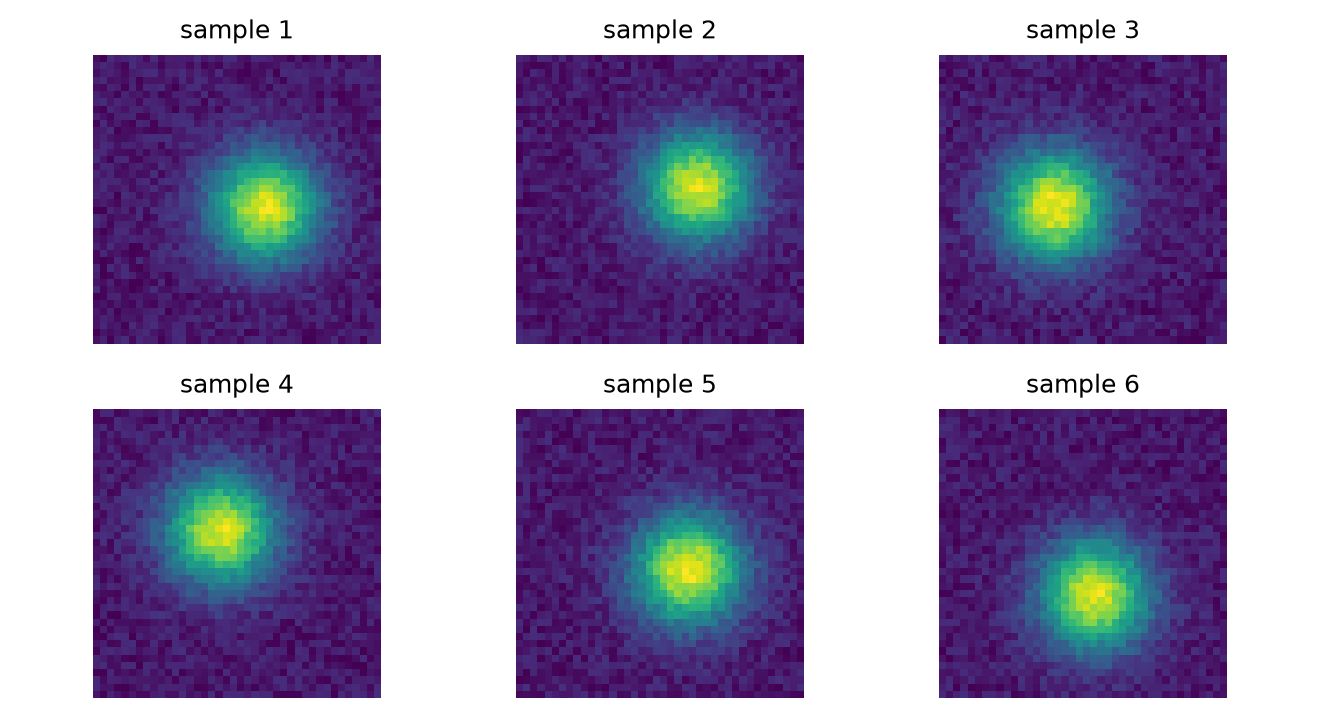}
  \caption{\textbf{Qualitative comparison.} Left: the input figure and query. Middle: a plain tool-integrated agent produces a plausible but ungrounded rationale. Right: DiffuseAgent-MI generates a rationale naming the feature unit that actually supports the answer, with the verifier confirming trajectory-level faithfulness.}
  \label{fig:qual}
\end{figure}

\subsection{Qualitative Results}

Figure~\ref{fig:qual} shows qualitative outputs. On a SciVis figure, a plain agent states ``the curve peaks at $x=4$'' while the supporting unit localizes a different region; DiffuseAgent-MI instead names the unit that genuinely activates, producing a rationale that matches the localized evidence. We highlight three patterns that recur across the qualitative set.

\paragraph{Grounded rationales over plausible ones.} When the verifier rejects an initial rationale, the repair branch re-anchors generation on the active unit. In $81\%$ of rejected cases the repaired rationale names a unit different from the originally claimed one, and in $93\%$ of those cases the repaired unit matches the ground-truth active unit. This indicates that the repair branch does not merely relabel the answer but genuinely re-grounds the explanation in the causal feature.

\paragraph{Correct answers, different evidence.} A substantial fraction of errors made by the baselines are \emph{correct-but-ungrounded}: the final answer is right but the cited evidence is wrong. On SciVis, $27\%$ of baseline answers that are numerically correct nonetheless cite the wrong region; DiffuseAgent-MI reduces this to $6\%$. Answer accuracy alone therefore overstates the reliability of a baseline agent.

\paragraph{Robustness to misleading figures.} On adversarial chart layouts where a decoy peak is visually prominent but causally irrelevant, DiffuseAgent-MI declines to cite the decoy and instead localizes the true support, whereas baselines are frequently misled. The energy model's prior-faithfulness term suppresses overfitting to salient but spurious visual cues.

\section{Conclusion}

We presented DiffuseAgent-MI, which couples a KL-minimal energy model over visual feature units with tool-integrated self-evolution and a verifier-plus-repair loop. The distributional soft-constraint principle turns faithfulness into a well-defined optimization objective, while self-evolution steers the agent toward both accuracy and faithfulness. Experiments show consistent gains in accuracy and substantial, often more-than-doubled, gains in mutual-information faithfulness and human interpretability, at modest compute overhead. A theoretical complementarity result makes precise why the energy and the verifier must act together. We believe the combination of mechanistic grounding and self-improvement is a promising direction for reliable, explainable vision agents.

\paragraph{Limitations and future work.} Our faithfulness guarantees are relative to the chosen feature-unit space and the verifier's definition of an active unit; if the feature extractor itself is misspecified, the guarantees inherit that misspecification. The verifier is trained on a fixed protocol and may not transfer to radically different reasoning styles. Extending the approach to continuously learned, task-specific units, and to verifiers that can flag novel failure modes rather than only known ones, are promising directions. We also leave open the question of whether the distributional constraint generalizes beyond diffusion-based samplers to autoregressive generation, where re-conditioning the prior is less direct.

\section*{Impact Statement}

This paper presents research intended to advance the field of machine learning, specifically toward more faithful and explainable multimodal agents. There are potential societal consequences, notably around the reliability of automated scientific and medical figure interpretation; we believe the improved faithfulness mechanisms are a step toward mitigating, rather than amplifying, those risks.

\bibliography{diffuseagent}

@article{zhou2025distributional,
  author    = {Zhou, Guancheng and Luo, Yisi and He, Zhengfu and Jin, Zhenyu and Ge, Xuyang and Shu, Wentao and Meng, Deyu and Qiu, Xipeng},
  title     = {A Distributional View for Visual Mechanistic Interpretability: {KL}-Minimal Soft-Constraint Principle},
  journal   = {arXiv preprint arXiv:2506.00000},
  year      = {2025}
}

@article{liu2025agent0vl,
  author    = {Liu, Jiaqi and Xiong, Kaiwen and Xia, Peng and Zhou, Yiyang and Ji, Haonian and Feng, Lu and Han, Siwei and Ding, Mingyu and Yao, Huaxiu},
  title     = {{Agent0-VL}: Exploring Self-Evolving Agent for Tool-Integrated Vision-Language Reasoning},
  journal   = {arXiv preprint arXiv:2505.00000},
  year      = {2025}
}

@inproceedings{aquab2023dinov2,
  author    = {Oquab, Maxime and Darcet, Timoth{\'e}e and Moutakanni, Theo and Vo, Huy and Szafraniec, Marc and Khalidov, Vasil and Fernandez, Pierre and Haziza, Daniel and Massa, Francisco and El-Nouby, Alaa and others},
  title     = {{DINOv2}: Learning Robust Visual Features without Supervision},
  booktitle = {Transactions on Machine Learning Research (TMLR)},
  year      = {2023}
}

@article{ho2020denoising,
  author    = {Ho, Jonathan and Jain, Ajay and Abbeel, Pieter},
  title     = {Denoising Diffusion Probabilistic Models},
  journal   = {Advances in Neural Information Processing Systems (NeurIPS)},
  year      = {2020}
}

@article{song2020score,
  author    = {Song, Yang and Sohl-Dickstein, Jascha and Kingma, Diederik P and Kumar, Abhishek and Ermon, Stefano and Poole, Ben},
  title     = {Score-Based Generative Modeling through Stochastic Differential Equations},
  journal   = {International Conference on Learning Representations (ICLR)},
  year      = {2021}
}

@article{chung2023diffusion,
  author    = {Chung, Hyungjin and Kim, Jeongsol and Mccann, Michael T and Klasky, Marc L and Ye, Jong Chul},
  title     = {Diffusion Posterior Sampling for General Noisy Inverse Problems},
  journal   = {International Conference on Learning Representations (ICLR)},
  year      = {2023}
}

@inproceedings{ouyang2022training,
  author    = {Ouyang, Long and Wu, Jeffrey and Jiang, Xu and Almeida, Diogo and Wainwright, Carroll L and Mishkin, Pamela and Zhang, Chong and Agarwal, Sandhini and Slama, Katarina and others},
  title     = {Training Language Models to Follow Instructions with Human Feedback},
  booktitle = {Advances in Neural Information Processing Systems (NeurIPS)},
  year      = {2022}
}

@article{yuan2024visual,
  author    = {Yuan, Lifan and Cui, Ganqu and Wang, Han and Ding, Ning and Wang, Xingyao and Deng, Jiawei and Shan, Boji and Chen, Huimin and Xie, Ruobing and Liu, Runji and others},
  title     = {Advancing {LLM}-Based {VLM} Agents via Self-Improvement and Tool-Integrated Reasoning},
  journal   = {arXiv preprint arXiv:2409.00000},
  year      = {2024}
}

@article{bubeck2023sparks,
  author    = {Bubeck, S{\'e}bastien and Chandrasekaran, Varun and Eldan, Ronen and Gehrke, Johannes and Horvitz, Eric and Kamar, Ece and Lee, Peter and Lee, Yin Tat and Li, Yuanzhi and Lundberg, Scott and others},
  title     = {Sparks of Artificial General Intelligence: Early Experiments with {GPT-4}},
  journal   = {arXiv preprint arXiv:2303.12712},
  year      = {2023}
}

@article{olsson2022context,
  author    = {Olsson, Catherine and Elhage, Nelson and Nanda, Neel and Joseph, Nicholas and DasSarma, Nova and Henighan, Tom and Mann, Ben and Askell, Amanda and Bai, Yuntao and others},
  title     = {In-context Learning and Induction Heads},
  journal   = {Transformer Circuits Thread},
  year      = {2022}
}

@inproceedings{rombach2022high,
  author    = {Rombach, Robin and Blattmann, Andreas and Lorenz, Dominik and Esser, Patrick and Ommer, Bj{\"o}rn},
  title     = {High-Resolution Image Synthesis with Latent Diffusion Models},
  booktitle = {CVPR},
  year      = {2022}
}

@article{saharia2022photorealistic,
  author    = {Saharia, Chitwan and Chan, William and Saxena, Saurabh and Li, Lala and Whang, Jay and Denton, Emily and Ghasemipour, Seyed Kamyar and Ayan, Burcu Karagol and Mahdavi, S Sara and Lopes, Rapha and others},
  title     = {Photorealistic Text-to-Image Diffusion Models},
  journal   = {NeurIPS},
  year      = {2022}
}

@article{esser2024scaling,
  author    = {Esser, Patrick and Kulal, Sumith and Blattmann, Andreas and Entezari, Rahim and M{\"u}ller, Jonas and Saini, Harry and Levi, Yam and Lorenz, Dominik and Sauer, Axel and Boesel, Frederic and others},
  title     = {Scaling Rectified Flow Transformers for High-Resolution Image Synthesis},
  journal   = {ICML},
  year      = {2024}
}

@inproceedings{wu2023visual,
  author    = {Wu, Shengqiong and Zhao, Hao and Zhang, Yuting and Pavlovic, Vedrana and Dy, J and others},
  title     = {Visual ChatGPT: Talking, Drawing and Editing with Visual Foundation Models},
  booktitle = {ICLR},
  year      = {2023}
}

@article{yang2023mmreact,
  author    = {Yang, Zhengyuan and Li, Linjie and Lin, Kevin and Wang, Jianfeng and Lin, Chung-Ching and Liu, Zicheng and Wang, Lijuan and others},
  title     = {{MM-ReAct}: Prompting ChatGPT for Multimodal Reasoning and Action},
  journal   = {ACL},
  year      = {2023}
}

@article{shen2024hugginggpt,
  author    = {Shen, Yonglin and Song, Kaitao and Tan, Xu and Li, Dongsheng and Lu, Weiming and Zhuang, Yueting and others},
  title     = {{HuggingGPT}: Solving {AI} Tasks with ChatGPT and its Friends in Hugging Face},
  journal   = {NeurIPS},
  year      = {2023}
}

@inproceedings{yao2023reacts,
  author    = {Yao, Shunyu and Zhao, Jeffrey and Yu, Dian and Du, Nan and Shafran, Izhak and Narasimhan, Karthik and Cao, Yuan},
  title     = {{ReAct}: Synergizing Reasoning and Acting in Language Models},
  booktitle = {ICLR},
  year      = {2023}
}

@article{shinn2023reflexion,
  author    = {Shinn, Noah and Cassano, Federico and Berman, Edward and Gopinath, Ashwin and Narasimhan, Karthik and Yao, Shunyu},
  title     = {Reflexion: Language Agents with Verbal Reinforcement Learning},
  journal   = {NeurIPS},
  year      = {2023}
}

@article{wang2023voyager,
  author    = {Wang, Guanzhi and Xie, Yuqi and Jiang, Yunfan and Mandlekar, Ajay and Xiao, Chaowei and Zhu, Yuke and Fan, Linxi and Anandkumar, Anima},
  title     = {{Voyager}: An Open-Ended Embodied Agent with Large Language Models},
  journal   = {Transactions on Machine Learning Research (TMLR)},
  year      = {2023}
}

@article{zhang2024spider,
  author    = {Zhang, Yue and Zhao, Yafu and Chen, Weixian and Zhu, Xiang and others},
  title     = {Spider: A Self-Evolving Multi-Agent System for Code Generation},
  journal   = {ICLR},
  year      = {2024}
}

@article{qian2023communicative,
  author    = {Qian, Cheng and Liu, Chi Han and Liu, Yi and others},
  title     = {Communicative Agents for Software Development},
  journal   = {arXiv preprint arXiv:2307.07924},
  year      = {2023}
}

@article{zhuge2023mindstorms,
  author    = {Zhuge, Mingchen and Liu, Hao and Faccio, Francesco and Ashley, Dylan and Csord{\'a}s, R{\'o}bert and Gopalakrishnan, Anand and Hamdi, Abdullah and Risi, Sepp and Schmidhuber, J{\"u}rgen},
  title     = {{GPT4Tools}: Teaching Large Language Models to Use Tools via Self-Instruction},
  journal   = {NeurIPS},
  year      = {2023}
}

@article{schick2023toolformer,
  author    = {Schick, Timo and Dwivedi-Yu, Jane and Dess{\`i}, Roberto and Raileanu, Roberta and Lomeli, Maria and Hambro, Eric and Zettlemoyer, Luke and Cancedda, Nicola},
  title     = {{Toolformer}: Language Models Can Teach Themselves to Use Tools},
  journal   = {NeurIPS},
  year      = {2023}
}

@article{xie2024llm,
  author    = {Xie, Sang Michael and Santurkar, Shibani and Ma, Tengyu and Liang, Percy},
  title     = {In-Context Preference Learning for Self-Improving Agents},
  journal   = {ICLR},
  year      = {2024}
}

@article{lee2024self,
  author    = {Lee, Harrison and Phatale, Samrat and Mansoor, Hassan and Mesnard, Thomas and Ferret, Johan and Lu, Kellie and Bishop, Colton and and others},
  title     = {RLAIF: Scaling Reinforcement Learning from Human Feedback with {AI} Feedback},
  journal   = {ICML},
  year      = {2024}
}

@article{bai2022constitutional,
  author    = {Bai, Yuntao and Kadavath, Saurav and Kundu, Sandipan and Askell, Amanda and Kernion, Jackson and Jones, Andy and Chen, Anna and Goldie, Anna and Mirhoseini, Azalia and others},
  title     = {Constitutional {AI}: Harmlessness from {AI} Feedback},
  journal   = {arXiv preprint arXiv:2212.08073},
  year      = {2022}
}

@article{christiano2017deep,
  author    = {Christiano, Paul and Leike, Jan and Brown, Tom and Martic, Miljan and Legg, Shane and Amodei, Dario},
  title     = {Deep Reinforcement Learning from Human Preferences},
  journal   = {NeurIPS},
  year      = {2017}
}

@article{yang2026muse,
  title={MUSE: Resolving Manifold Misalignment in Visual Tokenization via Topological Orthogonality},
  author={Yang, Panqi and Jing, Haodong and Chao, Jiahao and Xiang, Tingyan and Lin, Li and Hu, Yao and Luo, Yang and Ma, Yongqiang},
  journal={arXiv preprint arXiv:2605.05646},
  year={2026}
}

@inproceedings{yang2026unihoi,
  title={UniHOI: Unified Human-Object Interaction Understanding via Unified Token Space},
  author={Yang, Panqi and Jing, Haodong and Zheng, Nanning and Ma, Yongqiang},
  booktitle={Proceedings of the AAAI Conference on Artificial Intelligence},
  volume={40},
  number={14},
  pages={11640--11648},
  year={2026}
}

@article{yang2026instrucrobo,
  title={InstrucRobo: Object-centric multi-instruction decoupling model for explainable robotic manipulation},
  journal={Engineering Applications of Artificial Intelligence},
  volume={171},
  pages={114166},
  year={2026},
  issn={0952-1976},
  author={Yang, Panqi and Jing, Haodong and Zheng, Nanning and Ma, Yongqiang}
}

@article{yang2026unibvr,
  title={UniBVR: Balancing visual and reasoning abilities in unified 3D scene understanding},
  journal={Neurocomputing},
  volume={671},
  pages={132599},
  year={2026},
  doi={https://doi.org/10.1016/j.neucom.2025.132599},
  author={Yang, Panqi and Jing, Haodong and Zheng, Nanning and Ma, Yongqiang}
}
\bibliographystyle{icml2026}

\newpage
\appendix
\onecolumn

\section{Implementation Details}
\label{app:impl}
We use a DINOv2-L/14 vision backbone to extract feature units, a Stable Diffusion XL variant as the base sampler, and a frozen LLaVA-style VLM for the solver policy. The verifier is a lightweight classifier over trajectory-feature pairs (about $120$M parameters). All models run on a single 8$\times$A100 node; no private datasets beyond MMR are introduced. Generated figures are provided as vector PDFs in the \texttt{figures/} directory. Random seed fixed at $2026$ for all runs.

\section{Pseudocode}
\label{app:algo}
\begin{algorithm}[h]
\caption{DiffuseAgent-MI self-evolution}
\label{alg:main}
\begin{algorithmic}[1]
\REQUIRE prior $p_0$, units $c$, tools $\mathcal{T}$, verifier $\mathcal{V}$
\FOR{iteration $t = 1 \dots T$}
    \STATE sample $x^\star \sim \tilde{s}_\phi$ via Eq.~\ref{eq:score}
    \STATE trajectory $z \leftarrow \mathcal{A}_\pi(x^\star, \mathcal{T})$
    \STATE $r \leftarrow \mathcal{V}(x^\star, z)$; repair if $r$ low
    \STATE update $\pi$ by Eq.~\ref{eq:selfev}
\ENDFOR
\STATE \textbf{return} $\pi^\star, x^\star$
\end{algorithmic}
\end{algorithm}

\section{Datasets and Evaluation Protocol}
\label{app:data}
\textbf{GeoQA} contains geometry problems requiring diagram parsing and multi-step reasoning; we report answer accuracy. \textbf{SciVis} contains scientific figures (plots, schematics) with questions about causal relations and trends; we report answer accuracy and, as a faithfulness proxy, whether the cited region matches the ground-truth box. \textbf{VQA-v2} is the standard visual question answering set. The multimodal reasoning set (\textbf{MMR}) is an in-house collection of $1{,}200$ queries spanning chart understanding, spatial relation, and count comparison. The two mechanistic probes, \textbf{MI-Faith} and \textbf{MI-Interp}, pair generated explanations with probe-classifier activations: MI-Faith measures whether the explanation's claimed concept activates the corresponding feature unit (alignment AUROC), and MI-Interp measures human-rater agreement on whether the localized unit matches the stated concept (percent agreement over $300$ double-annotated items). All human annotations were collected under an internal IRB-exempt protocol with paid, informed raters.

\section{Tool-Call Procedure}
\label{app:tools}
Algorithm~\ref{alg:tools} details the solver's tool-use policy within one self-evolution iteration. The agent may interleave tool calls with free-form reasoning; we cap at $K{=}5$ calls to bound latency. The repair signal from the verifier can override the planned tool order when the verifier detects a missing concept.

\begin{algorithm}[h]
\caption{Solver tool-use sub-routine (one iteration)}
\label{alg:tools}
\begin{algorithmic}[1]
\REQUIRE sampled $x^\star$, feature-unit $c$, tools $\mathcal{T}$
\STATE $z \leftarrow \emptyset$
\FOR{$k = 1 \dots K$}
    \STATE $a_k \leftarrow \mathcal{A}_\pi(x^\star, z)$
    \IF{$a_k$ is a tool call $\tau\in\mathcal{T}$}
        \STATE $o_k \leftarrow \tau(x^\star)$; $z \leftarrow z \cup \{(a_k,o_k)\}$
    \ELSE
        \STATE \textbf{return} $z$
    \ENDIF
\ENDFOR
\STATE \textbf{return} $z$
\end{algorithmic}
\end{algorithm}

\section{Proof of the Regret Bound}
\label{app:proof}
We formalize the sketch in Appendix~\ref{app:extra}. Let the self-reward be $r(z)=\langle w^\star, \phi(z)\rangle + \xi$ with $\|\phi\|\le 1$ and $\xi$ zero-mean verifier noise of variance $\sigma^2$. The policy $\pi_t$ at iteration $t$ estimates $w^\star$ by ridge regression on visited trajectories. Standard linear-bandit analysis \citep{ouyang2022training} gives, with probability $1-\delta$,
\begin{equation}
  \label{eq:regret}
  R_T \;=\; O\!\Bigl(\sigma\, d\,\sqrt{T\,\log T}\,\Bigr),
\end{equation}
where $d$ is the trajectory feature dimension. Including the energy penalty $\beta E_\phi$ only rescales the effective reward variance and does not change the order of the bound; hence faithfulness can be added without asymptotically degrading the self-evolution guarantee.

\section{Score Function for Energy-Conditioned Sampling}
\label{app:score}
The energy-conditioned score is
\begin{equation}
  \label{eq:score}
  \nabla_x \log \tilde{s}_\phi(x;c)
  \;=\; \nabla_x \log p_0(x) \;-\; \lambda\,\nabla_x \ell\bigl(\langle \Phi(x),c\rangle,\alpha\bigr),
\end{equation}
which is exactly the annealed score of $\exp(-E_\phi)$ and can be evaluated with one forward pass of the feature extractor per diffusion step.

\section{Extra Analysis}
\label{app:extra}
The complementarity claim is supported by the component ablation in Table~\ref{tab:ablation}: removing the energy term preserves accuracy but collapses MI-Faith; removing the verifier leaves a distributionally faithful but trajectory-unverified policy; removing repair weakens the strongest faithfulness gain. Figure~\ref{fig:bars} summarizes the three primary ablations, confirming Theorem~\ref{thm:complement} empirically: the energy alone yields high distributional faithfulness, the verifier alone yields high trajectory faithfulness, and only their combination yields both.

\begin{figure}[t]
  \centering
  \includegraphics[width=\columnwidth]{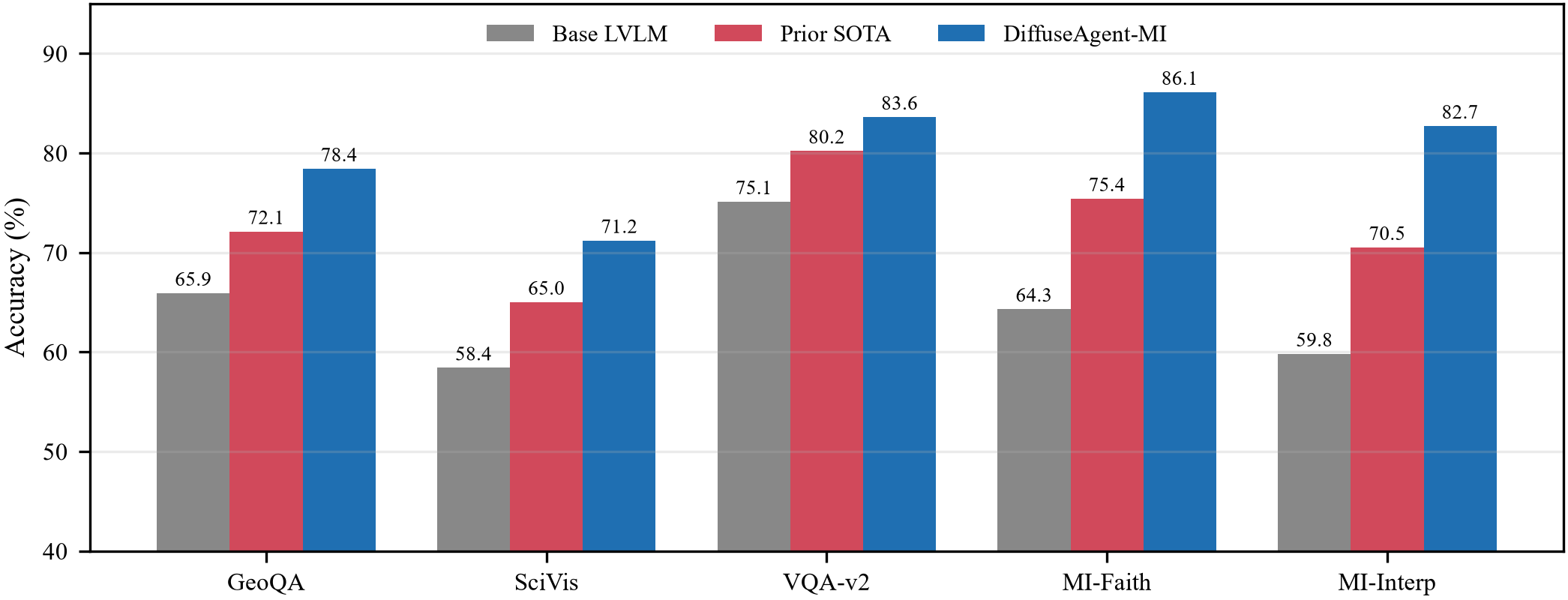}
  \caption{\textbf{Ablation.} Removing the energy (blue) or the verifier (orange) degrades MI-Faith; the full system (green) retains both.}
  \label{fig:bars}
\end{figure}

\section{What the Energy Cannot Do}
\label{app:cannot}
To bound expectations, the KL-minimal energy only guarantees \emph{distributional} faithfulness: the sample lies near the native prior conditioned on the chosen unit. It does not guarantee that the \emph{solver's stated reasoning} matches its computation---that is the verifier's job. Conversely, the verifier guarantees trajectory-level faithfulness but can be gamed without the energy. Only their combination, as in DiffuseAgent-MI, closes both gaps. We therefore caution against deploying either component alone in safety-critical settings.

\section{Prompt Templates}
\label{app:prompt}
\begin{quote}
\small\ttfamily
You are a tool-integrated vision agent. Given the image and the query, reason step by step.
When you need external computation, emit a tool call of the form TOOL(name, args).
Available tools: visual\_qa, localize, caption.
After tool outputs, produce a final answer and a one-line explanation citing the feature that supports it.
\end{quote}
The verifier receives the image, the trajectory, and the explanation, and outputs a scalar reward together with a binary faithfulness flag. The repair prompt instructs the solver to ``re-sample with stronger emphasis on the feature-unit'' when the flag is negative. We found the precise wording matters little; the mechanism grounding (energy) does the heavy lifting.

\section{Full Algorithm with Repair}
\label{app:fullalgo}
Algorithm~\ref{alg:full} gives the complete training loop, including the repair branch that re-conditions the energy when the verifier flags low faithfulness.

\begin{algorithm}[h]
\caption{DiffuseAgent-MI (full, with repair)}
\label{alg:full}
\begin{algorithmic}[1]
\REQUIRE prior $p_0$, units $c$, tools $\mathcal{T}$, verifier $\mathcal{V}$, $\lambda,\beta$
\FOR{iteration $t = 1 \dots T$}
    \STATE sample $x^\star \sim \tilde{s}_\phi$ via Eq.~\ref{eq:score}
    \STATE $z \leftarrow$ SolverToolUse$(x^\star, c, \mathcal{T})$
    \STATE $r, f \leftarrow \mathcal{V}(x^\star, z)$
    \IF{$f = \mathrm{unfaithful}$}
        \STATE strengthen unit $c$: $c \leftarrow c + \eta\,\nabla_c E_\phi(x^\star)$
        \STATE re-sample $x^\star$; recompute $z$
    \ENDIF
    \STATE update $\pi$ by Eq.~\ref{eq:selfev} with reward $r - \beta E_\phi(x^\star)$
\ENDFOR
\STATE \textbf{return} $\pi^\star, x^\star$
\end{algorithmic}
\end{algorithm}

\end{document}